\documentclass[acmsmall,screen,nonacm]{acmart}
\setcopyright{none}
\makeatletter
\@ifundefined{hyxmp@parse@acmart}{}{\let\hyxmp@parse@acmart\relax}
\makeatother
\usepackage{amsmath,mathtools,booktabs,tabularx,array,enumitem,microtype}
\usepackage{tikz}
\usetikzlibrary{arrows.meta,positioning,fit,calc}
\hypersetup{keeppdfinfo=true,hidelinks,pdftitle={Verifying Text-to-SQL Queries under Changing Business Semantics},pdfauthor={Genliang Zhu; Xiaobo Chen; Ziyuan Wang; Chu Wang},pdfcreator={LaTeX}}
\setlist[itemize]{leftmargin=1.25em}
\setlist[enumerate]{leftmargin=1.5em}
\newcolumntype{Y}{>{\raggedright\arraybackslash}X}
\newcommand{\execute}{\operatorname{Exec}}
\newcommand{\metricvalue}{\operatorname{Metric}}
\newcommand{\nonconf}{\textsc{Nonconforming}}
\newcommand{\notrefuted}{\textsc{Not Refuted}}
\newcommand{\inconclusive}{\textsc{Inconclusive}}
\newcommand{\bounded}{\textsc{Bounded Conformance}}
\newtheorem{proposition}{Proposition}
\newif\ifControlledResults
\ControlledResultstrue
\newcommand{\ControlledCases}{24}
\newcommand{\ControlledLedgerChecks}{912}
\newcommand{\ControlledReferenceMatches}{792}
\newcommand{\ControlledReferenceChecks}{792}
\newcommand{\ControlledInvalidRejected}{18}
\newcommand{\ControlledBankLedgers}{38}
\newcommand{\ControlledClusters}{14}
\newcommand{\ControlledBad}{10}
\newcommand{\ControlledFiniteAgreement}{14}
\newcommand{\ControlledUnknown}{0}
\newcommand{\ControlledGenericDetected}{6}
\newcommand{\ControlledGuidedDetected}{10}
\newcommand{\ControlledIdenticalSets}{0}
\newcommand{\ControlledIntersectionTotal}{121}
\newcommand{\ControlledUnionTotal}{263}

\newcommand{\NatTasks}{96}
\newcommand{\NatCells}{144}
\newcommand{\NatBad}{104}
\newcommand{\NatAgree}{33}
\newcommand{\NatUnresolved}{7}
\newcommand{\NatLatent}{43}
\newcommand{\NatSnapshot}{61}
\newcommand{\NatGeneric}{95}
\newcommand{\NatGuided}{101}

\newcommand{\NatPool}{102}

\newcommand{\NatJudgeGptFR}{29}

\newcommand{\NatJudgeGemmaFR}{32}
\newcommand{\NatCarriedBad}{42}
\newcommand{\NatCarriedCells}{48}

\newcommand{\NatCoincideSemBad}{19}

\newcommand{\NatCoincideSemSnapshot}{8}
\newcommand{\NatGenCalls}{118}
\newcommand{\NatRepairs}{22}
\newcommand{\NatJudgeCalls}{288}

\newcommand{\NatCrossAgree}{4752}
\newcommand{\NatCrossCompared}{4752}
\newcommand{\NatChecks}{1478}
\newcommand{\NatGuidedVsGeneric}{5.8 [2.0, 9.7]}
\newcommand{\NatGuidedVsSnapshot}{38.5 [25.3, 51.5]}
\newcommand{\NatGenericVsSnapshot}{32.7 [19.4, 46.4]}
\newcommand{\NatClusters}{23}
\newcommand{\NaturalLabelTable}{\begin{table}[t]\centering\small
\caption{Independent finite-suite labels of final natural SQL on the 24 frozen test families (374 legal ledgers per family; test families were not used for prompt tuning). ``D0-latent'' counts nonconforming cells whose value agrees with the definition on the current snapshot.}\label{tab:natural-labels}
\begin{tabular}{llrrrrr}\toprule Model & Cell & Cells & Nonconf. & D0-latent & Agree & Unres.\\\midrule
gpt-oss-20b & old$\to$v1 & 24 & 20 & 10 & 3 & 1\\
gpt-oss-20b & old$\to$v2 & 24 & 22 & 9 & 1 & 1\\
gpt-oss-20b & new$\to$v2 & 24 & 17 & 10 & 6 & 1\\
gemma-4-12b & old$\to$v1 & 24 & 14 & 4 & 8 & 2\\
gemma-4-12b & old$\to$v2 & 24 & 20 & 6 & 2 & 2\\
gemma-4-12b & new$\to$v2 & 24 & 11 & 4 & 13 & 0\\
\midrule All & all & 144 & 104 & 43 & 33 & 7\\\bottomrule\end{tabular}\end{table}}
\newcommand{\NaturalDetectionTable}{\begin{table}[t]\centering\small
\caption{Detection of labeled nonconforming cells by cell type and false rejection of finite-suite-agreement cells (test split). Ledger arms never reject an agreeing cell by construction of replayable evidence; judges can.}\label{tab:natural-detection}
\begin{tabular}{lrrrrr}\toprule Method & old$\to$v1 & old$\to$v2 & new$\to$v2 & All & False rej.\\\midrule
Nonconforming / agreeing cells & 34 & 42 & 28 & 104 & 33\\\midrule
Current snapshot $D_0$ & 20 & 27 & 14 & 61 & 0\\
Generic legal (8) & 30 & 42 & 23 & 95 & 0\\
Obligation-guided (8) & 33 & 42 & 26 & 101 & 0\\
Current-obligation (8) & 33 & 42 & 26 & 101 & 0\\
Equal-weight (8) & 33 & 42 & 26 & 101 & 0\\
Whole selection pool (74) & 33 & 42 & 27 & 102 & 0\\
LLM judge gpt-oss-20b & 33 & 42 & 28 & 103 & 29\\
LLM judge gemma-4-12b & 34 & 42 & 28 & 104 & 32\\
\bottomrule\end{tabular}\end{table}}

\newcommand{\NatNoPrevGuided}{101}
\newcommand{\NatNoPrevGeneric}{96}
\newcommand{\NatNoPrevContrast}{4.8 [0.9, 10.0]}
\newcommand{\NatSampCells}{432}
\newcommand{\NatSampTasks}{288}
\newcommand{\NatSampBad}{305}
\newcommand{\NatSampAgree}{109}
\newcommand{\NatSampUnres}{18}
\newcommand{\NatSampLatent}{129}
\newcommand{\NatSampSnapshot}{176}
\newcommand{\NatSampGeneric}{266}
\newcommand{\NatSampGuided}{285}

\newcommand{\NatSampGuidedVsGeneric}{6.2 [2.9, 10.0]}
\newcommand{\NatSampGuidedVsSnapshot}{35.7 [28.5, 44.0]}
\newcommand{\NatSampCalls}{353}
\newcommand{\NatSampGroups}{144}
\newcommand{\NatSampMixed}{65}
\newcommand{\NatSampAllBad}{67}
\newcommand{\NatSampNoneBad}{12}

\newcommand{\NatSampMajorityMatch}{102}

\newcommand{\NatJGptCheckDzeroDet}{68}
\newcommand{\NatJGptCheckDzeroFR}{1}

\newcommand{\NatJGemmaCheckFR}{25}

\newcommand{\NaturalJudgeTable}{\begin{table}[t]\centering\small
\caption{LLM-judge variants on the 144 test cells (104 nonconforming, 33 agreeing). Variants were tuned on development families only; the primary variant per judge was fixed by a written dev rule (maximum detection minus false-rejection rate, ties to the simpler prompt) before test verdicts were run. Ledger checks are shown for reference.}\label{tab:natural-judges}
\begin{tabular}{lrrrl}\toprule Reviewer & Detected & False rej. & Abstain & Dev-selected\\\midrule
gpt-oss-20b, base prompt & 103 & 29 & 0 & \\
gpt-oss-20b, checklist & 48 & 4 & 0 & \\
gpt-oss-20b, checklist + $D_0$ execution & 68 & 1 & 0 & yes\\
gemma-4-12b, base prompt & 104 & 32 & 0 & \\
gemma-4-12b, checklist & 100 & 25 & 4 & yes (tie)\\
gemma-4-12b, checklist + $D_0$ execution & 102 & 24 & 3 & \\
\midrule Current snapshot $D_0$ & \NatSnapshot{} & 0 & -- & \\ Obligation-guided (8 ledgers) & \NatGuided{} & 0 & -- & \\
\bottomrule\end{tabular}\end{table}}

\newcommand{\ConfFamilies}{32}
\newcommand{\ConfCells}{192}
\newcommand{\ConfBad}{144}
\newcommand{\ConfAgree}{36}
\newcommand{\ConfUnresolved}{12}
\newcommand{\ConfSnapshot}{81}
\newcommand{\ConfGeneric}{127}
\newcommand{\ConfGuided}{137}
\newcommand{\ConfDifference}{6.9}
\newcommand{\ConfCI}{[2.1, 13.6]}
\newcommand{\ConfP}{0.0156}
\newcommand{\ConfLatent}{63}
\newcommand{\ConfGenCalls}{153}
\newcommand{\ConfRepairs}{25}
\newcommand{\ConfTruncated}{3}
\newcommand{\ConfCompareSQL}{6336}
\newcommand{\ConfCarriedBad}{56}
\newcommand{\ConfCarriedCells}{64}

\title{Verifying Text-to-SQL Queries under Changing Business Semantics}
\author{Genliang Zhu}
\affiliation{\institution{Accentrust}\city{Vancouver}\country{Canada}}
\affiliation{\institution{Georgia Institute of Technology}\department{School of Computer Science}\city{Atlanta}\state{Georgia}\country{USA}}
\email{research@accentrust.com}
\author{Xiaobo Chen}
\affiliation{\institution{University of Massachusetts Amherst}\city{Amherst}\state{Massachusetts}\country{USA}}
\author{Ziyuan Wang}
\affiliation{\institution{Duke University}\city{Durham}\state{North Carolina}\country{USA}}
\author{Chu Wang}
\affiliation{\institution{Accentrust}\city{Vancouver}\country{Canada}}
\affiliation{\institution{University of Illinois Urbana-Champaign}\department{Department of Computer Science}\city{Urbana}\state{Illinois}\country{USA}}

\begin{document}
\begin{abstract}
In text-to-SQL and AI-agent analytics, successful SQL execution does not establish semantic correctness under changing business definitions. We study revalidation of fixed queries after changes to refund attribution, tax treatment, eligible populations and exchange-rate dates. An executable verification profile binds the requested revision and knowledge cutoff to legal counterexample data and exact monetary comparison. A preregistered confirmation on \ConfFamilies{} newly authored revision families and two local producer bundles retains \ConfCells{} query--definition cells: \ConfBad{} nonconforming, \ConfAgree{} finite-suite agreeing and \ConfUnresolved{} unresolved. Of the nonconforming cells, \ConfLatent{} agree on the current snapshot. At an eight-ledger budget, obligation-guided checking detects \ConfGuided{} discrepancies versus \ConfGeneric{} for generic ordering, a paired difference of \ConfDifference{} percentage points (family bootstrap 95\% interval \ConfCI; exact family sign-flip $p=\ConfP$). Every ledger arm rejects zero of the \ConfAgree{} agreeing cells. Current-obligation and equal-weight scoring match the guided result, identifying coverage rather than revision-specific priority as the observed source of benefit. Controlled assays, a natural-query pilot and repeated samples provide separate mechanism and robustness evidence. A legal discrepancy refutes conformance; agreement remains restricted to the recorded suite. The study establishes an auditable business-semantic revalidation profile within its authored schema and local-model scope.
\end{abstract}
\maketitle

\section{Introduction}
Natural-language analytics translates a business question into a query and presents its result. Enterprise questions frequently depend on knowledge outside a schema: which transactions count, when a refund is recognized, whether an amount includes tax, and which date selects an exchange rate. EntSQL directly studies enterprise knowledge grounding~\cite{liao2026entsql}. Spider 2.0 studies realistic enterprise text-to-SQL workflows~\cite{lei2025spider2}, and BIRD emphasizes database content and efficient execution~\cite{li2023bird}. These capabilities are necessary for useful analytics, but successful execution alone does not establish that a query implements the requested definition.

Consider a January report with knowledge available through February. One approved definition subtracts refunds in their occurrence month; another explicitly restates the original sale month. A database with no refunds yields the same January value under both definitions. A legal February refund against a January sale separates them. A query filtering refunds to January can therefore agree with the new definition on the current database while implementing the old definition. This concerns query meaning, not a claim that the current numeric answer has already become wrong.

Definition updates make the distinction operational. Historical reports, cached SQL, and ad-hoc compositions may survive an update. A new revision does not make every old query invalid: changes can affect an unrelated metric or documentation only, and a user may explicitly request a historical definition. Revalidation must identify the requested interpretation and admissible data before declaring a discrepancy.

Counterexample databases and SQL equivalence evaluation are established research areas. Distilled test suites~\cite{zhong2020testsuite}, XData~\cite{chandra2015xdata}, ParSEval~\cite{chen2025parseval}, and SpotIt~\cite{klopfenstein2026spotit} address accidental agreement. SpotIt+ incorporates database constraints~\cite{tremante2026spotitplus}. Building on these foundations, we study how explicit business-definition transitions, temporal applicability and business legality determine revalidation of an already fixed query.

The contribution is business-semantic revalidation of an already fixed query, with explicit evidence about what each tested component adds:
\begin{enumerate}
\item A correctness contract binds query, requested business revision, event time, knowledge cutoff and legal input domain. It preserves valid historical requests and distinguishes a replayable discrepancy from finite agreement or an unresolved result.
\item An executable temporal/monetary profile constructs legal test obligations from approved definitions and checks them against a separately implemented evaluator. An old/new query--definition matrix distinguishes semantic revisions, documentation-only updates and unrelated metrics.
\item An empirical decomposition separates generic ordering, current-definition coverage and revision-priority scores. A natural-query pilot, repeated samples and a registered new-family confirmation show coverage benefits under a common budget; matching current-only and equal-weight results identify the measured component.
\end{enumerate}
The profile integrates established counterexample testing, exact arithmetic, constraint-aware verification and greedy coverage around the business-revision contract. Table~\ref{tab:claim-evidence} maps its contract, implementation and empirical findings to their evidence. In the controlled assay the coverage orders coincide; in the confirmation their eight-ledger detection counts coincide, while shorter prefixes can differ.

Direct compilation remains an important alternative. If a trusted metric layer completely expresses a request, executing its approved definition is preferable to asking a model to rediscover it. Our setting concerns existing SQL, heterogeneous interfaces, and compositions requiring revalidation. Native metric compilation and checking an already fixed query are therefore distinct evaluation paths.

\begin{table}[t]
\caption{Claims, executed evidence and interpretation. Qualification and future designs do not enlarge the empirical denominator.}
\label{tab:claim-evidence}
\centering\small
\begin{tabularx}{\linewidth}{p{0.20\linewidth}YY}
\toprule
Claim & Executed evidence & Interpretation\\
\midrule
Revision-bound conformance & Separate policy evaluator, legality checks and old/new query--definition matrix & Legal discrepancy refutes the requested profile; passing tests do not prove full-domain conformance.\\
Coverage component & Controlled decomposition; natural-query pilot and repeated samples & Current-only/equal-weight results distinguish coverage from priority weighting.\\
Local confirmation & 32 new families, two bundles, 192 cells; eight-ledger paired comparison & 137/144 versus 127/144 detection; authored-schema/local-bundle scope.\\
Native support & Original ParSEval, PV-SQL and MetricFlow qualification paths & Support measurements, not an effectiveness ranking of complete systems.\\
\bottomrule
\end{tabularx}
\end{table}

\paragraph{Relation to strong alternatives.}
Distilled suites and ParSEval test SQL equivalence using generated databases; SpotIt+ searches for distinguishing instances under database constraints~\cite{zhong2020testsuite,chen2025parseval,tremante2026spotitplus}. This profile treats the approved business definition and its temporal applicability as explicit verification inputs for an existing query. The approaches can complement one another: a qualified reference compilation and business-constraint mapping could connect the revision profile to native equivalence checking. Reported native outcomes are limited to the interfaces actually qualified; an unexecuted mapping remains an integration proposal, rather than a baseline failure.

\section{Background and Related Work}
\subsection{Enterprise SQL and governed metrics}
EntSQL makes enterprise knowledge an explicit input~\cite{liao2026entsql}; Spider 2.0 expands the workflow setting~\cite{lei2025spider2}; BIRD supplies heterogeneous database grounding~\cite{li2023bird}. DIN-SQL decomposes generation and self-correction~\cite{pourreza2023dinsql}. These studies inform query production and external data selection. They do not, by themselves, provide labels for every business-policy revision introduced by another study.

Governed definitions reduce informal interpretation. GROUND uses structured definitions and checks~\cite{pillai2026ground}; semantic-layer-mediated agents use an intermediate representation and compilation~\cite{kim2026semanticagent}. These approaches motivate explicit population, grain, amounts, dates, and dependencies, with metadata for temporal applicability and verification scope.

A semantic interchange format has a schema-format version; a metric has a business revision. A format upgrade can leave the metric unchanged, while a policy change can occur without changing the format. The requested business revision determines verification. These fields must not be conflated.

\subsection{Tests, equivalence, and constraints}
Distilled test suites evaluate SQL on multiple databases~\cite{zhong2020testsuite}. XData distinguishes query mutations~\cite{chandra2015xdata}; ParSEval uses plan-aware generation~\cite{chen2025parseval}. SpotIt uses formal verification~\cite{klopfenstein2026spotit}. SpotIt+ v3 combines bounded equivalence, configurable domain constraints, and distinguishing databases~\cite{tremante2026spotitplus}. Our development profile studies approved revisions and knowledge-time/event-consumption contrasts. The experiments separate budget-matched authored verification controls from native-tool qualification; a native comparative evaluation requires the same temporal and business constraints.

Cosette~\cite{chu2017cosette}, U-semiring~\cite{chu2018usemiring}, and VeriEQL~\cite{he2024verieql} distinguish equivalence, counterexamples, and bounded reasoning. Bag multiplicities, NULL, ordering, and integrity constraints affect results. A finite testing budget cannot inherit an unbounded guarantee merely because a solver is present.

\subsection{Verification and repair}
SQLDriller studies execution consistency and correction~\cite{yang2025sqldriller}. PV-SQL combines database probing with deterministic checks and refinement~\cite{tian2026pvsql}. Its verifier is relevant to fixed-query checks; its full probing and repair path is a different arm with additional costs.

TraceSQL estimates answerability without gold SQL or reference execution results at inference, while using business and schema signals~\cite{shukla2026tracesql}. Our procedure instead uses approved executable reference semantics, which specify a metric without requiring manually authored gold SQL for every query. The scientific evaluator and the operational reference are implemented separately.

\begin{table}[t]
\caption{Closest-work boundaries and evaluation obligations. Established components are credited rather than presented as new foundations.}
\label{tab:related}
\begin{tabularx}{\linewidth}{@{}lYY@{}}\toprule
Work & Established object & Obligation here \\\midrule
EntSQL / Spider 2.0 & Enterprise knowledge and workflows & Explicit revision interventions \\
Metric-layer agents & Governed representation and compilation & Strong native compilation path \\
XData / ParSEval & Distinguishing test databases & Same constraints and budget \\
SpotIt / SpotIt+ & Verification and constraints & Temporal and business controls \\
Cosette / VeriEQL & Equivalence and bounded reasoning & Explicit verification scope \\
PV-SQL / SQLDriller & Checks, probing, and correction & Fixed checks versus full repair \\
TraceSQL & Reference-free estimation & Executable reference dependence \\\bottomrule
\end{tabularx}
\end{table}

\section{Problem Definition}
\subsection{Requests and temporal applicability}
A request $a$ contains a metric identifier, explicit business revision $v$, report interval $[t_0,t_1)$, knowledge cutoff $k$, and applicable restatement policy. The interval includes its start and excludes its end. The approved definition $C_v$ specifies grain, population, statuses, event-time basis, amounts, currencies, exchange-rate date, and rounding. A historical revision remains valid when explicitly requested; verification does not silently substitute the latest revision.

Knowledge time and event time are separate. An event contributes only after occurring and being recorded by $k$. A refund known after the report period can alter that period when the approved definition permits restatement. Otherwise, later information cannot be assumed to revise a historical answer.

Ambiguity or revision mismatch yields an unresolved outcome. Business approval of a definition is a trust premise, not a conclusion from agreement between implementations. The study uses explicit research policies rather than claiming universal accounting treatment.

\subsection{Legal data and execution}
Let $\mathcal{D}(C_v,a)$ be databases satisfying schema and business constraints. Legality requires more than keys: refunds reference transactions and obey amount and funding bounds; canceled sales cannot create collected or refunded funds; rates cover consumed currency/date pairs; money and rounding are exact.

The execution profile fixes dialect, supported SQL, grain, bag/ordering semantics, NULL handling, and comparison. Scalar money uses base-currency minor units. Each event is converted using a positive rational rate, rounded half-up to a minor unit, and then added or subtracted. Rounding the final sum is a different policy.

A warehouse can contain future or late-recorded events while remaining legal. The cutoff excludes such events from consumption, rather than rejecting a normal historical query over a current warehouse.

\subsection{Correctness objects}
Let $\execute(q,D)$ be deterministic query execution and $\metricvalue(C_v,a,D)$ the approved metric. The relation $\simeq$ is a fixed result comparison. Current-answer agreement is
\begin{equation}
\execute(q,D_0)\simeq\metricvalue(C_v,a,D_0).
\label{eq:current}
\end{equation}
Query conformance is stronger:
\begin{equation}
\forall D\in\mathcal{D}(C_v,a),\quad \execute(q,D)\simeq\metricvalue(C_v,a,D).
\label{eq:conformance}
\end{equation}
A counterexample is a legal $D^\star$ where comparison fails. Agreement on a finite $\mathcal{T}\subseteq\mathcal{D}(C_v,a)$ establishes agreement on $\mathcal{T}$, not Equation~\eqref{eq:conformance} on the full domain.

\begin{proposition}
If the metric evaluator, legality check, execution, and comparison faithfully implement their semantics, a replayable legal instance with unequal results refutes conformance under the requested revision.
\end{proposition}
\begin{proof}
The conformance statement quantifies over every legal instance. The stated instance belongs to that domain and violates the required comparison.
\end{proof}
This applies the standard witness argument to the requested business revision. Admissibility is required for a valid rejection, and a search that finds no discrepancy retains its finite verification scope.

\subsection{Definition transitions}
A transition creates obligations relative to the requested interpretation. Refund attribution can affect refund queries; unrelated metrics and documentation changes may not. A change in legal-data constraints is distinct from a change in the metric expression. A database admissible under only one version cannot automatically serve as a common version witness.

The executed development matrix crosses old/new authored reference queries with both definitions for nine transitions or controls on the same legal ledger bank. It includes historical requests, unrelated changes, and metadata-only revisions. Its monetary and structural legality profile remains unchanged: it does not test a change in database constraints or naturally migrated enterprise queries. A changed identifier can invalidate evidence binding, but does not prove that the query's meaning is wrong.

\section{Verification Procedure}
\subsection{Normalization and test obligations}
The implemented scalar profile represents approved population, event attribution, amount, FX, and rounding with typed fields. A broader representation can also express grain and dependencies, but their general extraction is not implemented here. SQL execution reuses SQLite's native parser and authorizer; these safety checks are distinct from SQL-AST normalization.

An obligation identifies a difference that data can expose. Sales versus receipts separates sale/payment dates or permits partial payment. Refund attribution places a funded refund across a boundary. Tax creates distinct gross/net values. FX supplies different approved rates on relevant dates. Composed obligations preserve the same funding, amount, and temporal rules.

A candidate pool can be prepared from definitions before test SQL is seen. The revised development implementation synthesizes matched amount, population, event-date, equal/unequal-rate, rounding, and knowledge-cutoff ledgers from public current/previous policies and requests. It adds shared base snapshots and a deterministic legal sampler, but does not import the authored mutations' witness ledgers. Selection reads neither query text, mutation tags, family identifiers nor scientific labels. The templates remain author-designed development constructions, not independent external tasks. Query-guided search is a separate, unimplemented extension.

\subsection{Budgeted legal tests}
The selector evaluates supported public policy contrasts on each candidate ledger. An obligation is covered only if consumed metric values differ: an unused date's changed rate is insufficient for an FX obligation. Single-field old/new contrasts and a complete previous-policy contrast receive priority; alternate policy, cutoff, endpoint, per-event rounding, and consumed-sale multiplicity diagnostics provide general coverage. A joint contrast retains coupled changes that cannot form a supported single-field policy. These operational evaluations do not use the scientific evaluator and their wall time is charged to selection.

Greedy selection assigns four units to an uncovered changed-policy obligation and one to another uncovered obligation, accounting for coverage in the base snapshot. A fixed hash resolves ties. The original seed and 4/1 weights are retained, rather than fitted to favorable outcomes. Two matched ablations use the identical bank, budget, tie rule and operational capabilities: current-obligation coverage substitutes the current policy for the previous policy only in obligation scoring, removing changed contrasts from that score; equal-weight revision coverage retains the contrasts but assigns every obligation weight one. The common bank is built from both public policies, and the shared tie key still encodes both policies and the request. Thus neither ablation removes previous-definition information from the entire pipeline. Their order and execution outcomes are recorded separately. This fixed-count procedure has no cost-optimality guarantee. Collection includes selected-set overlap, available versus uncovered obligations, missed discrepancies, and nested one/two/four/eight-ledger execution prefixes. Budget curves reaggregate charged runs, rather than new free searches.

Legality is checked before execution. Invalid candidates retain their cost and reason. An unfunded refund, missing exchange rate, or illegal duplicate is not a successful semantic detection.

\subsection{Read-only execution and decisions}
Execution uses isolated synthetic or licensed research databases, read-only access, bounded time, and an explicit operation policy. Parsing errors, unsupported outputs, timeouts, and non-scalar results are recorded rather than translated to zero.

The verifier returns \nonconf{} with a replayable discrepancy, \notrefuted{} with finite scope, or \inconclusive{} with a reason. \bounded{} is available only for an actually complete enumeration or qualified solver profile, not a passing random test.

Evidence binds request, definition, schema/dialect, query, ledger, expected/observed value, comparison, and budget. Content hashes preserve those bindings for replay and reproduction. Primary scientific labels are independently generated and frozen before arm execution. Selected inputs are fixed without reading those labels. A second independent replay checks the frozen operational decisions afterward, without revising primary labels or feeding back into search.

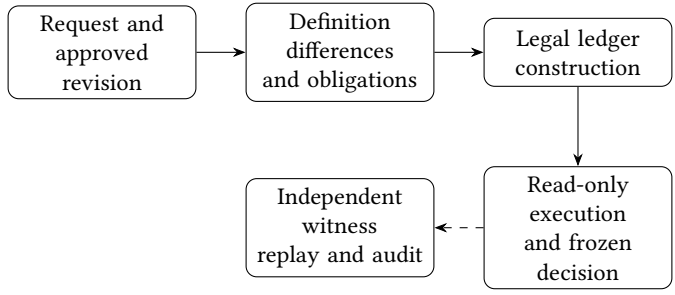
\begin{figure}[t]
\centering
\begin{tikzpicture}[font=\small,box/.style={draw,rounded corners,align=center,minimum height=0.9cm,text width=2.25cm},>=Stealth]
\node[box] (r) {Request and\\approved revision};
\node[box,right=0.65cm of r] (o) {Definition differences\\and obligations};
\node[box,right=0.65cm of o] (l) {Legal ledger\\construction};
\node[box,below=1.05cm of l] (d) {Read-only execution\\and frozen decision};
\node[box,left=0.65cm of d] (a) {Independent witness\\replay and audit};
\draw[->] (r.east)--(o.west);
\draw[->] (o.east)--(l.west);
\draw[->] (l.south)--(d.north);
\draw[->,dashed] (d.west)--(a.east);
\end{tikzpicture}
\caption{Primary labels are frozen independently before arm execution and remain hidden. The final replay checks frozen decisions without altering those labels or guiding search.}
\Description{Approved request and revision lead to definition differences, legal ledgers, and read-only execution. Independent replay subsequently checks the frozen operational decision; primary labels were independently frozen before this path.}
\label{fig:procedure}
\end{figure}

\subsection{Revalidation and compilation}
Relevant changes default to revalidation. Reuse requires a complete semantic dependency footprint and unchanged applicable request, constraints, dialect, and test scope. Omitting reuse does not imply that every revision requires rejection.

A native compiler provides a strong generation path. On fixed historical SQL, direct metric execution on $D_0$ provides a strong current-answer check. Counterexamples ask whether that query remains compatible away from $D_0$. Faithful compilation is not disadvantaged by withholding definitions or scoring unsupported ad-hoc requests as its failures.

\section{Implementation and Independent Validation}
\subsection{Policy profile}
The reference uses sales, payments, refunds, and date-indexed rational rates. Sales support net/gross amounts and posting/original-period refunds. Payment metrics use gross receipts with occurrence-period refunds; they do not infer tax decomposition of partial payments.

All ISO dates, identifiers, relationships, amounts, cumulative refunds, and funding at refund dates are checked. Identifier/revision strings require a non-whitespace character under the explicit ECMAScript whitespace set. Approved statuses are nonempty and duplicate-free with exact case-sensitive population matching; cancellation aliases are separately recognized case-insensitively and prohibit payment/refund events. Same-day payments fund same-day refunds. Base currency explicitly supplies a unit rate on consumed dates. BigInt interpretation returns unresolved on output overflow; arbitrary candidate SQL need not perform exact arithmetic.

The authored reference SQL uses SQLite integer arithmetic within the controlled small-value domain. Generated ledgers contain at most five sales, each with net amount at most 20,000 and tax at most 4,000 minor units, and rates from $1/2,3/4,1,5/4,3/2$. Its conversion expression has intermediate magnitude at most $2(24{,}000)(5)+4=240{,}004$. Each rounded event is at most 36,000 minor units and aggregate absolute magnitude is at most 720,000. Thus these reference expressions do not overflow SQLite integers; the argument is not extended to arbitrary safe-integer input magnitudes.

\subsection{Separately implemented evaluator}
The scientific evaluator computes policy rules event by event using decimal arithmetic and independently checks legality. It imports neither operational evaluator, generated reference SQL, nor core decision logic. Input facts and policy meaning are shared; executable adjudication logic is not.

Hand-calculated cases cover cross-period refunds, late recording, partial receipts, FX dates, half-unit rounding, invalid funding, and historical requests. Separately authored JS/Python vectors additionally cover nonempty/unique statuses, Unicode whitespace, strict revisions, cancellation aliases, invalid dates and numeric types. Agreement establishes cross-implementation consistency on these specified vectors. The two implementations belong to the same research project, and business approval of the underlying definitions remains a separate trust premise.

\subsection{Label freezing and alignment}
Queries are fixed before adjudication. The independent labeling procedure and resources are specified before detection runs. Primary labels are frozen blind to arm outputs. New arm-discovered evidence is a common secondary analysis, not a selective change to primary denominators.

Each label carries its data/SQL profile. Finite-domain agreement does not label a query globally correct. False rejection uses the same domain; an outside-domain witness is a scope-extension finding, not a false rejection of bounded conformance. Current-reference match alone does not establish conformance. Unresolved labels and detector abstention are separate.

\section{Controlled Validation}
\subsection{Objects and controls}
The assay uses authored reference queries and semantic mutations, not model outputs. Policy and data vary separately. A mutation can coincide with its reference on the base ledger, while negative controls leave semantics unchanged.

The five controlled arms are authored reference controls, not reproduced ParSEval, SpotIt+, PV-SQL, or MetricFlow systems. All share the same request, current/previous public policies, schema, deduplicated ledger bank and tie key. The single-database arm compares the current snapshot and is diagnostic. Generic, priority-weighted, current-obligation and equal-weight revision controls each execute eight distinct ledgers, including that snapshot; they differ in public-input selection scores. Every selected execution is retained, even after a discrepancy. No cache supplies an arm with uncharged comparison results. Current-obligation coverage omits changed-policy contrasts from its score, not from bank preparation or tie-breaking. End-to-end removal of previous-definition information was not executed in this controlled assay; its separately executed natural-query pilot control is reported in Section~\ref{sec:natural-robust}.

The revised bank includes policy-generated paired controls, public base snapshots and deterministic random ledgers. It excludes authored mutation witness ledgers as a bank source. The generator/coverage revision was motivated by the first development assay and editorial diagnostics, so this is a follow-up mechanism assay, not a blinded natural-error study or confirmation of superiority. The full finite bank labels only this recorded suite, not the legal business domain. Selection preparation, scientific labeling, and operational execution are recorded separately. The proposed strong combination remains unexecuted.

\ifControlledResults
\subsection{Recorded outcomes}
The assay contains \ControlledCases{} authored cases in \ControlledClusters{} scenario clusters and \ControlledBankLedgers{} deduplicated ledgers, giving \ControlledLedgerChecks{} query--ledger adjudication pairs. Operational and scientific metric values agree on \ControlledReferenceMatches{} of \ControlledReferenceChecks{} legal comparisons. Both implementations reject \ControlledInvalidRejected{} directed invalid inputs, including a revision mismatch as well as illegal ledgers. The full suite labels \ControlledBad{} queries unequal, \ControlledFiniteAgreement{} agreeing within the suite, and \ControlledUnknown{} unresolved.
\begin{table}[t]
\caption{Authored controlled assay. Detection uses 10 finite-suite unequal queries; rejection is evaluated on 14 finite-suite agreement queries, not globally proved queries. Time includes selection and all operational executions on this host. The current-snapshot control uses one ledger; the other controls use eight.}
\label{tab:controlled}
\begin{tabularx}{\linewidth}{@{}Yrrrrr@{}}\toprule
Control & Detected & Rejected & Unresolved & SQL calls & Seconds \\\midrule
Current snapshot only & 0/10 & 0/14 & 0 & 24 & 0.963 \\
Generic legal bank order & 6/10 & 0/14 & 1 & 192 & 7.525 \\
Definition-guided bank order & 10/10 & 0/14 & 0 & 192 & 7.810 \\
Current-obligation cover & 10/10 & 0/14 & 0 & 192 & 8.100 \\
Equal-weight revision cover & 10/10 & 0/14 & 0 & 192 & 7.602 \\
\bottomrule\end{tabularx}
\end{table}
\begin{table}[t]
\caption{Nested operational execution prefixes on the same development queries. Each detection denominator is 10; SQL calls are per arm. This is a budget manipulation check, not independent confirmation. All failures and full eight-ledger costs remain recorded.}
\label{tab:budget}
\begin{tabularx}{\linewidth}{@{}Yrrrrr@{}}\toprule
Ledgers & Generic & Guided & Current cover & Equal weight & SQL calls \\\midrule
1 & 0/10 & 0/10 & 0/10 & 0/10 & 24 \\
2 & 3/10 & 8/10 & 8/10 & 8/10 & 48 \\
4 & 4/10 & 10/10 & 10/10 & 10/10 & 96 \\
8 & 6/10 & 10/10 & 10/10 & 10/10 & 192 \\
\bottomrule\end{tabularx}
\end{table}

The revised generic and definition-guided orders detect \ControlledGenericDetected{} and \ControlledGuidedDetected{} of \ControlledBad{} suite-labeled discrepancies, respectively. The current-snapshot control detects none. Their selected sets are identical for \ControlledIdenticalSets{}/\ControlledCases{} queries, with pooled intersection/union counts of \ControlledIntersectionTotal{}/\ControlledUnionTotal{}; this ratio is descriptive, not an independent-sample estimate. At eight ledgers the guided arm covers all obligations available in this bank. The generic arm misses the FX, late-recording, aggregate-rounding and inclusive-end mutations; the rounding decision remains unresolved after an empty-set NULL output. SQL failures remain recorded even where another selected ledger establishes a discrepancy.

The current-obligation and equal-weight controls each select the identical full eight-ledger order for all \ControlledCases{} queries and detect 10/10 discrepancies. All three coverage orders have the same 0/8/10/10 detections at one/two/four/eight ledgers. Within the shared bank and tie rule, the suite supports obligation scoring over generic ordering, but provides no incremental detection advantage for changed contrasts in that score or 4/1 change priority. Each coverage arm has two recorded SQL failures, although other selected ledgers establish the discrepancies. Selection costs in this run are approximately 95\,ms priority-weighted, 80\,ms current-obligation, 88\,ms equal-weight, and 9\,ms generic; total arm times include every execution (Table~\ref{tab:controlled}). These are single-host measurements, not randomized timing comparisons.

The earlier frozen assay (001) contained 31 ledgers and 744 adjudication pairs, with both orders detecting 8/10 and one unresolved decision each. Its selected sets were identical for 20/24 queries, and the FX mutation was missed by both. Assay 002 used the same authored queries with revised bank/coverage and observed 10/10 versus 6/10; assay 003 adds the decomposition controls without changing those queries, the bank or weighted selection. All historical results remain archived. The three development rounds cannot be pooled as independent evidence or interpreted as held-out improvement estimates. Scientific full-bank labeling is common evaluation overhead, reported separately from operational arm execution.
\begin{table}[t]
\caption{Independently labeled examples: candidate SQL / requested metric, in base-currency minor units. Witnesses come from full-suite scientific labeling, not necessarily the eight-ledger operational selection.}
\label{tab:witnesses}
\begin{tabularx}{\linewidth}{@{}Yrr@{}}\toprule
Definition difference & Current ledger & Legal witness \\\midrule
Sales / receipts & 100/100 & 100/0 \\
Refund period & 100/100 & 100/80 \\
Tax treatment & 100/100 & 120/100 \\
FX date & 150/150 & 60/180 \\
\bottomrule\end{tabularx}
\end{table}

These measurements establish recorded policy/witness behavior and identify which tested component accounts for the observed detection, not natural agent error prevalence or production effectiveness. No population-level interval or native-method superiority claim is inferred from the authored fourteen-cluster development set.
% Derived from results/version-matrix-002; independent manifest SHA256 18834bcb74ac716a795ea5e7ab1e1ee0b4d2b6a4c64470b2ed95e7890d3126a1
\subsection{Old/new query--definition cross-product}
The separate development matrix contains 9 transition/control groups and 38 ledgers. Two physical SQL executions per group/ledger are compared to both requested definitions: 684 executions and 1368 logical query--definition comparisons, not 1368 SQL calls. The 684 unique policy/ledger values agree with separate Python adjudication. Both revision-mismatch directions are rejected before SQL dispatch in all 18 checks.
\begin{table}[t]
\caption{Unequal results among 38 legal ledgers per cell. Columns denote query/definition versions (O: old; N: new). Zero means agreement in this finite suite, not domain-wide equivalence. The legal-data profile is unchanged across versions.}
\label{tab:version-matrix}
\begin{tabularx}{\linewidth}{@{}Yrrrr@{}}\toprule
Transition/control & O/O & O/N & N/O & N/N \\\midrule
Sales / receipts & 0 & 15 & 15 & 0 \\
Refund attribution & 0 & 11 & 11 & 0 \\
Tax treatment & 0 & 20 & 20 & 0 \\
FX date & 0 & 3 & 3 & 0 \\
Eligible statuses & 0 & 3 & 3 & 0 \\
Combined changes & 0 & 21 & 21 & 0 \\
Metadata only & 0 & 0 & 0 & 0 \\
Unrelated metric & 0 & 0 & 0 & 0 \\
Historical request & 0 & 11 & 11 & 0 \\
\bottomrule\end{tabularx}
\end{table}
Matched-version cells agree throughout the finite bank. Off-diagonal discrepancies expose the specified semantic changes; metadata and unrelated-metric controls retain agreement. The historical group explicitly requests the old revision rather than treating every old query as invalid. The mismatched-request checks test evidence binding separately from semantic disagreement. These are authored reference queries on the original development bank, not natural migration, changed-constraint or new-schema transfer evidence.

\else
\subsection{Accounting}
The assay records every query, legality decision, result, operational decision, independent value, and failure. Comparative quantities derive from recorded evidence. Authored cases remain separate from natural-query evaluation.
\fi

\subsection{Failures}
Generation and execution failures retain their original denominators. Unknown policies, revision mismatch, unsupported shape, invalid data, and resource exhaustion are separated. Rejecting everything does not demonstrate diagnostic effectiveness. Disagreement between policy implementations blocks verified-evidence release pending clarification; frozen labels and attempts remain preserved rather than silently rewritten.

\subsection{Upstream qualification}
We separately executed fixed-source ParSEval and PV-SQL verification components without model calls. This is tool qualification, not a full business-policy comparison. ParSEval produced a distinguishing threshold database that was independently replayed in SQLite. A refund-join diagnostic returned equality despite a separately replayable difference. The latter outcome is retained as a missed discrepancy in that tested profile, not interpreted as proof of equivalence. A cross-table refund bound also cannot be expressed by a direct SQLite CHECK subquery; it requires an explicit business constraint implementation. Consequently, a qualified schema-only test does not certify the full monetary and temporal domain.

A subsequent business adapter invokes the original generator on both frozen candidate/reference SQL and applies the same public ledger legality check before accepting a discrepancy. Eight authored queries across the four core policy families were each given a 2,000\,ms shared qualification budget. All eight outcomes were unresolved: generated sale dates included malformed dates or unsupported timestamp forms, so no captured candidate entered legal SQL replay within the budget. These discarded candidates, generator starts, and timeouts remain recorded. This result describes adapter support in that profile; it is not an effectiveness ranking or a completed full-constraint baseline.

We also qualified the native MetricFlow v0.213.0 engine at fixed source revision \texttt{4200f85}.%
\footnote{Official source: \url{https://github.com/dbt-labs/metricflow/tree/v0.213.0}; Apache-2.0 license. Software provenance is separate from the scholarly reference count.}
Its original semantic-manifest transformation, engine, and DuckDB SQL-plan renderer generated the executed queries, rather than an author-written metric compiler. Two sales requests with date and status restrictions produced net and gross totals of 300 and 360 integer minor units, respectively; both matched the separate policy implementation. This qualification covers native sales aggregation in that adapter, not refund attribution, payments, late knowledge, or non-identity exchange rates. Native query generation is a separate track from checking a fixed historical SQL query, and these two cases do not enter the controlled detection denominator.

\section{Natural Local-Model Pilot}
\label{sec:natural}
This pilot replaces authored mutations with SQL written by language models. It remains an authored-family exploratory study: revision families, requests and ledgers were constructed by the authors, prompts were tuned on a separate development split, and no confirmatory hypothesis was registered.

\subsection{Families, candidates and labels}
Each family pairs revisions v1 and v2 of one metric under one request. Twenty test families change refund attribution, tax mode, FX date, eligible statuses, basis, or several fields jointly; four are controls that change only documentation or an unrelated catalog entry. A current snapshot $D_0$ is drawn from a seeded legal generator using only the public definitions. Semantic test families alternate between snapshots on which v1 and v2 coincide and snapshots on which they differ, assigned by position before any model call. Four further development families were used only for prompt and parser checks.

For every family, a model writes one query for an explicit v1 request (\emph{old}) and one for v2 (\emph{new}). The prompt contains the target definition only, the schema, the request, the SQLite sandbox restrictions and a preview of $D_0$; it never mentions the other revision. Two model--runtime bundles were qualified: \texttt{openai/gpt-oss-20b} through the OpenAI-compatible chat interface with default reasoning, and \texttt{google/gemma-4-12b-qat} through raw completion with its canonical template and thinking disabled. Two further local models produced no SQL within 6k--12k reasoning tokens and were excluded at qualification. Decoding uses temperature 0. A task allows the first answer and at most one repair, triggered only by a public execution failure, NULL, or missing SQL on $D_0$. The \NatTasks{} test tasks used \NatGenCalls{} calls (\NatRepairs{} repairs), with one truncated response and no transport failure.

Development exposed two prompt problems, fixed before the test split was generated. Nearly every development query applied SQLite integer division before rounding; a dialect note alone did not change this behavior. The final prompt therefore states the approved per-event conversion expression. Rounding is identical in both revisions and is not the semantics under study, so this supplies a governed formula rather than the answer to a revision question. All three development rounds are archived.

Primary cells are old$\to$v1 (historical request), old$\to$v2 (carried query after adoption of v2), and new$\to$v2, giving \NatCells{} test cells. The Python evaluator and a Python SQLite executor label every cell on 374 legal ledgers per family: the 74-ledger pool from which checkers select, and 300 further ledgers never offered to any checker. A legal difference labels a cell nonconforming. Agreement on every ledger is finite-suite agreement, and execution failures without a difference are unresolved. Labels were hashed before any checker or judge ran.

\subsection{Checkers and judges}
The five ledger arms reuse the controlled-validation implementation without change: the snapshot, a generic legal order, the obligation-guided order, and its current-obligation and equal-weight ablations, each executing eight ledgers from the 74-ledger pool (the snapshot executes $D_0$ only). The previous revision supplied to selection is v1 for all cells. The LLM-judge baseline gives each bundle the same definition, request, schema, snapshot preview and SQL, plus the previous definition for carried cells, and requests a binary verdict. Each judge makes one call per primary cell with no retry, \NatJudgeCalls{} calls in total; malformed verdicts count as abstentions.

An independent checker verified freeze hashes, the complete task denominator, the attempt rule, binding of labels to the exact SQL, recomputation of every arm decision from its executions, that every arm rejection has a nonconforming label, and judge completeness (\NatChecks{} checks, all passing). The JavaScript runtime and the Python evaluator/executor agree on all \NatCrossAgree{} of \NatCrossCompared{} arm executions.

\subsection{Results}
\NaturalLabelTable
Table~\ref{tab:natural-labels} shows that natural errors are common under explicit definitions: \NatBad{} of \NatCells{} cells are nonconforming, \NatAgree{} agree throughout the suite and \NatUnresolved{} are unresolved. \NatLatent{} nonconforming cells return the defined value on $D_0$, so a current-answer check cannot see them. After adoption of v2, \NatCarriedBad{} of \NatCarriedCells{} carried queries are nonconforming. Among the twenty carried semantic-change cells whose families have coinciding snapshots, \NatCoincideSemBad{} are nonconforming and the snapshot detects only \NatCoincideSemSnapshot{}.

\NaturalDetectionTable
Table~\ref{tab:natural-detection} reports detection. The snapshot detects \NatSnapshot{} of \NatBad{}, the generic order \NatGeneric{}, and the obligation-guided order \NatGuided{}. The whole 74-ledger pool reaches \NatPool{}; the remaining two discrepancies occur only on label-only ledgers. No ledger arm rejects an agreeing cell, because every rejection carries a replayable legal witness. Paired over the nonconforming cells and resampled by family (\NatClusters{} families containing such cells), the guided minus generic detection difference is \NatGuidedVsGeneric{} percentage points with a 95\% bootstrap interval, guided minus snapshot is \NatGuidedVsSnapshot{}, and generic minus snapshot is \NatGenericVsSnapshot{}. The current-obligation and equal-weight ablations reproduce the guided detections exactly, confirming on natural queries that the benefit comes from obligation coverage rather than revision-specific scoring.

With the original prompt, both LLM judges label nearly every query as nonconforming, including \NatJudgeGptFR{} (gpt-oss-20b) and \NatJudgeGemmaFR{} (gemma-4-12b) of the \NatAgree{} agreeing cells. Their high detection therefore reflects a near-constant verdict rather than discrimination, and their reasons are not replayable evidence. Section~\ref{sec:natural-robust} reports development-calibrated judges.

A descriptive post-label diagnosis neutralizes one data feature at a time. The most frequent defect is arithmetic over separately aggregated sales and refund subtotals: with no qualifying refunds, the refund subtotal is NULL, the difference becomes NULL, and an outer \texttt{COALESCE} returns zero. Further defects involve refund filtering, FX dates and foreign currency, case-variant statuses, and the knowledge cutoff. Carried queries additionally match the previous revision in eight cells. Many of these defects are invisible on a snapshot that happens to contain refunds, rates of one, or no late records.

\subsection{Stronger controls and repeated samples}
\label{sec:natural-robust}
\NaturalJudgeTable
\paragraph{Calibrated judges.} To give the judge baseline a fair chance, two revised prompts were tuned on the development families only: a definition-clause checklist that asks for rejection only with a concrete counterexample, and the same checklist plus the query result and the definition value on $D_0$, both of which are public to every checker. A written rule, fixed before any revised test verdict, selected one primary variant per judge by development detection minus false-rejection rate. Table~\ref{tab:natural-judges} shows the test outcomes. The selected gpt-oss-20b variant reduces false rejection to \NatJGptCheckDzeroFR{} of \NatAgree{} but detects only \NatJGptCheckDzeroDet{} of \NatBad{}, close to the snapshot check it partly contains. The selected gemma-4-12b variant still rejects \NatJGemmaCheckFR{} agreeing cells. Calibration therefore moves these judges along a detection/false-rejection trade-off without reaching the eight-ledger check, which detects \NatGuided{} with no false rejection and returns a replayable witness for each rejection.

\paragraph{Target-definition-only control.} The controlled ablations kept the revision-informed bank and tie key. An end-to-end control removes the previous revision entirely: bank templates and the tie key are derived from the target definition and request only, and the obligation score uses no changed-revision contrast. Because every family contributes both revisions as targets, the resulting 74-ledger pool coincides with the original one; selection orders still change for most cells. Obligation-guided detection remains \NatNoPrevGuided{} of \NatBad{} and generic detection is \NatNoPrevGeneric{}, both without false rejection (guided minus generic \NatNoPrevContrast{} percentage points). At the evaluated budget, target-definition-only coverage matches the revision-informed coverage's detection count on these recorded queries.

\paragraph{Repeated samples.} Each test task was regenerated three times at temperature 0.7 with fixed seeds (\NatSampTasks{} tasks, \NatSampCalls{} calls, no transport failure), giving \NatSampCells{} primary cells labeled with the same suite. \NatSampBad{} are nonconforming, \NatSampAgree{} agree and \NatSampUnres{} are unresolved; \NatSampLatent{} nonconforming cells agree on $D_0$. Snapshot, generic and obligation-guided checks detect \NatSampSnapshot{}, \NatSampGeneric{} and \NatSampGuided{}, with no false rejection (guided minus generic \NatSampGuidedVsGeneric{}; guided minus snapshot \NatSampGuidedVsSnapshot{}). Correctness varies across samples: of \NatSampGroups{} model/family/cell groups, \NatSampAllBad{} have three nonconforming samples, \NatSampNoneBad{} none, and \NatSampMixed{} are mixed; the temperature-0 label matches the three-sample majority in \NatSampMajorityMatch{} groups. Single-sample error rates should therefore be read as one draw from a variable producer.

An earlier attempt at this collection was stopped after concurrent requests for the other model made the local server evict the running model, so that 72 tasks failed with server errors. That record is preserved and excluded as an environment failure; the reported run was recollected under the same frozen protocol with one model active at a time.

\subsection{Pilot frame and inference}
The families, schema, requests and ledgers are authored. The two evaluated local-model bundles produce one deterministic sample per task in the primary analysis and three sampled answers in the robustness analysis. The 24 test families form the clustering unit; the 144 cells are not independent observations, and the intervals describe this family set, not a population of business definitions. Finite-suite agreement is not domain-wide correctness. The prompt supplies the approved conversion formula, which removes one natural error class by design. Native constraint-aware tools, external schemas and changed legal domains were not part of this pilot.

\section{Held-Out Local Confirmatory Study}
\label{sec:confirmation}
\subsection{Frozen design and chronology}
A dated, source-bound preregistration fixes the primary comparison before the held-out generation. Thirty-two new revision families H01--H32 were constructed independently of the pilot's prompt and judge tuning cases. They include 26 semantic revisions and six metadata-only or unrelated-metric controls, and add a two-month request window absent from the pilot test set. Both local bundles, the generation prompts, temperature 0, seed 26 and the eight-ledger checking budget are retained. Each of 128 generation tasks admits at most one repair after a public current-snapshot failure, producing \ConfCells{} primary query--definition cells. No LLM judge participates in this confirmation.

Two definitions (H19 and H25) were corrected before freeze because the generator found no snapshot satisfying their assigned coincidence stratum; final strata are 19 coinciding and 13 separating families. The selection pool contains 81 deduplicated legal ledgers, whereas the preregistration text retained the pilot's size of 74. This discrepancy is disclosed: the frozen deterministic bank construction determines the actual pool, and every arm receives the same 81 candidates and eight selected-test allowance. Each family additionally has 300 label-only ledgers. These pre-data corrections and the pool discrepancy are retained in the archived protocol.

The primary endpoint is detection among independently nonconforming cells. Paired guided-minus-generic differences are summed within revision family across both bundles and cell types; the registered two-sided exact sign-flip test uses the nonzero family sums at $\alpha=0.05$. The interval resamples whole families. The hypothesis concerns this authored family frame. The exact-test interpretation assumes sign symmetry under the null; it does not establish sampling from a population of businesses. Unresolved cells remain in the full study denominator.

\subsection{Results and independent replay}
\begin{table}[t]
\caption{Preregistered local confirmation: 192 cells from 32 new revision families and two model bundles. Detection is conditional on 144 finite-suite nonconforming cells; false rejection is evaluated on 36 agreeing cells. Twelve unresolved cells remain in the full denominator.}
\label{tab:heldout-confirmation}
\centering\small
\begin{tabular}{@{}lrr@{}}\toprule
Check & Detected / 144 & False rejection / 36 \\\midrule
Current snapshot & 81 & 0 \\
Generic legal order & 127 & 0 \\
Obligation-guided order & 137 & 0 \\
Current-obligation scoring & 137 & 0 \\
Equal-weight scoring & 137 & 0 \\
\bottomrule\end{tabular}
\end{table}

The complete frame contains \ConfBad{} nonconforming, \ConfAgree{} finite-suite-agreeing and \ConfUnresolved{} unresolved cells. Of the nonconforming cells, \ConfLatent{} agree on the current snapshot; \ConfCarriedBad{} of \ConfCarriedCells{} carried old queries are nonconforming under v2. The snapshot, generic and obligation-guided checks detect \ConfSnapshot{}, \ConfGeneric{} and \ConfGuided{}, respectively (Table~\ref{tab:heldout-confirmation}). All five arms have zero false rejection among the \ConfAgree{} agreeing cells. Guided-minus-generic detection is \ConfDifference{} percentage points, with a family bootstrap interval of \ConfCI{} and exact $p=\ConfP$: eleven cells favor guided checking, one favors generic checking and 132 tie. Seven families have nonzero sums, all positive.

Current-obligation and equal-weight scoring also detect \ConfGuided{}. Their equality preserves the component interpretation from the pilot: these data support obligation coverage under the tested budget and do not establish a gain from revision-specific scores or weights. The full selection pool detects 140 of 144 discrepancies; four require label-only ledgers.

The recorded generation uses \ConfGenCalls{} calls, \ConfRepairs{} repairs, \ConfTruncated{} truncated responses and zero transport errors. Truncations and unresolved outcomes remain recorded. The checker passes 1,970 checks and compares \ConfCompareSQL{} arm executions between the JavaScript operational interpreter and the independently implemented Python evaluator. During manuscript integration, the checker, paired aggregation and full finite-suite SQL labels are recomputed from the supplied raw records without new model calls.

\subsection{Recorded heterogeneity, budget and computation}
\label{sec:offline-confirmation}
The following descriptive analyses were added after the confirmation results were available. They reaggregate the same frozen labels and charged execution records; they do not replace the registered endpoint, choose new parameters or execute additional queries.
% Derived only from archived observations; post hoc, not a new experiment.
\begin{table}[t]
\caption{Confirmation stratified by producer and revision axis. Entries are detected/nonconforming cells; revision axes are family attributes, not inferred SQL error causes.}\label{tab:confirm-strata}
\centering\small
\begin{tabularx}{\linewidth}{Yrrr}
\toprule
Stratum & Snapshot & Generic & Coverage\\
\midrule
gpt-\allowbreak{}oss-\allowbreak{}20b & 42/80 & 70/80 & 78/80\\
gemma-\allowbreak{}4-\allowbreak{}12b & 39/64 & 57/64 & 59/64\\
tax & 11/18 & 17/18 & 18/18\\
refundAttribution & 6/19 & 14/19 & 18/19\\
fxDate & 12/22 & 20/22 & 20/22\\
statuses & 11/17 & 16/17 & 17/17\\
basis & 10/15 & 15/15 & 15/15\\
joint & 17/27 & 24/27 & 27/27\\
metadata\_control & 5/13 & 10/13 & 11/13\\
unrelated\_control & 9/13 & 11/13 & 11/13\\
\bottomrule
\end{tabularx}
\end{table}
\begin{table}[t]
\caption{Post-hoc detection at prefixes of the already executed confirmation orders. Counts use the same 144 nonconforming cells; no new SQL is executed.}\label{tab:confirm-budget}
\centering\small
\begin{tabularx}{\linewidth}{Yrrrr}
\toprule
Order & 1 & 2 & 4 & 8\\
\midrule
Snapshot & 81 & 81 & 81 & 81\\
Generic & 81 & 99 & 113 & 127\\
Obligation coverage & 81 & 118 & 129 & 137\\
Current-obligation & 81 & 117 & 129 & 137\\
Equal-weight & 81 & 118 & 129 & 137\\
\bottomrule
\end{tabularx}
\end{table}
\begin{table}[t]
\caption{Recorded confirmation checking work. All selected executions were charged, including those after a discrepancy. Times are shared-host sums, not randomized speed comparisons.}\label{tab:confirm-cost}
\centering\small
\begin{tabularx}{\linewidth}{Yrrr}
\toprule
Arm & Selection (s) & Execution (s) & Comparisons\\
\midrule
Snapshot & 0.00 & 11.15 & 192\\
Generic & 0.08 & 87.94 & 1536\\
Obligation coverage & 2.00 & 87.89 & 1536\\
Current-obligation & 1.73 & 88.25 & 1536\\
Equal-weight & 1.98 & 88.33 & 1536\\
\bottomrule
\end{tabularx}
\end{table}
\newcommand{\CovLooMin}{4.3}
\newcommand{\CovLooMax}{7.2}

Table~\ref{tab:confirm-strata} separates producer bundles and revision axes. Coverage detects 78/80 discrepancies from gpt-oss-20b and 59/64 from Gemma, versus 70/80 and 57/64 for generic ordering. The revision-axis rows classify the authored family, not the cause of a generated SQL error: even documentation-only controls can contain a query-generation error. The ten net additional detections span seven families; H07 contributes four. Removing each whole family in turn retains a positive paired difference, ranging from \CovLooMin{} to \CovLooMax{} percentage points. This sensitivity describes the recorded frame, not 32 new tests or population generalization.

The prefixes in Table~\ref{tab:confirm-budget} show where the fixed orders obtain evidence. At two ledgers, guided coverage detects 118 discrepancies versus 99 for generic ordering; at four, the counts are 129 versus 113. Current-obligation coverage detects 117 at two and matches 129 and 137 at four and eight. Thus one shorter prefix differs, while the registered eight-ledger result still identifies general coverage rather than an extra revision-priority benefit. No budget was selected from this curve, and the shorter prefixes are not independent experiments.

Table~\ref{tab:confirm-cost} includes the same 192 cells in every arm, not only the 144 labeled nonconforming cells. Coverage preparation costs more than generic ordering, while all four bank-order arms charge 1,536 executions. The 6,336 total comparisons include the 192 snapshot comparisons. Generation used 153 calls and approximately 1,849.8 seconds of summed response latency; scientific labeling took 32.28 seconds in its own stage. These are separate resource categories, not disjoint wall-clock components to sum into an end-to-end deployment price. Shared bank construction is common preparation; finite-suite labeling is evaluation overhead, not an input given to operational checking. A smaller selected prefix does not imply proportional total savings because preparation is still needed.

\subsection{Protocol chronology and recorded differences}
\begin{table}[t]
\caption{Confirmation chronology. Times are recorded local UTC$-07{:}00$ timestamps, not independent timestamp attestations. Source bindings and all original records are retained.}
\label{tab:confirmation-chronology}
\centering\small
\begin{tabularx}{\linewidth}{p{0.22\linewidth}YY}
\toprule
Object & Frozen/executed record & Consequence\\
\midrule
H19/H25 definitions & Registration documents pre-freeze corrections after no compatible assigned snapshot was found & Final 19 coinciding/13 separating strata; no outcome-conditioned replacement is reported.\\
Protocol and generation & Protocol creation: 12 Oct., 14:36:40; first producer START: 16:59:13 & Recorded ordering places protocol before collected generation.\\
Candidate bank & Registered text: 74+300 ledgers; executed pool: 81+300 & A literal size discrepancy; all arms share the 81-candidate pool and eight-ledger limit.\\
Labels and checking & Label digest recorded before arm START at 17:30:55 & Labels remain independent of selection and all 192 primary cells remain in scope.\\
Editorial reaggregation & Added 13 Oct. from existing records only & Stratification, prefixes and leave-one-family-out are post hoc; registered inference is unchanged.\\
\bottomrule
\end{tabularx}
\end{table}
The 81-ledger bank follows the bound deterministic construction, but differs from the written size. Disclosure of the common bank does not erase that deviation. Source snapshots, protocol, generation events and label bindings let a reader distinguish the pre-data corrections from subsequent descriptive analysis. No raw chronology record was revised during manuscript editing.

\paragraph{External scope.} The local confirmation completes the authored-family study, not native full-constraint ranking or independent-schema transfer. The prospective extension is kept separate in Appendix~\ref{app:external-extension}; no part of that design is counted as an executed result.

\section{Discussion}
\subsection{Interpreting discrepancies}
A legal witness identifies a failure on its instance and can localize date, amount, or population mismatch. It does not necessarily make the present answer wrong. Current agreement is useful evidence about $D_0$, not every admissible warehouse state.

\subsection{Definition fidelity and application}
Definitions require business approval. If policy is wrong or underspecified, comparison cannot establish business truth. Disagreement calls for clarification, not silently edited rules.

A generic interface returns result, revision, state, and evidence. Rejection requests correction; unresolved cases request clarification; finite agreement displays scope. Existing business/authorization processes decide publication. The verifier neither authorizes more access nor consumes client schemas.

\subsection{Validity of comparisons}
Authored mutations measure mechanism coverage; natural queries measure another distribution. External schemas with authored revisions do not alone establish naturally occurring drift. Separate strata give each conclusion an identifiable source.

The empirical contribution is obligation coverage under a fixed verification budget, together with an auditable temporal/monetary revalidation profile. In the authored suite and the registered confirmation, current-obligation scoring matches the eight-ledger detection of revision-priority scoring, conditional on the shared bank and tie key. The previous-information-free pilot control preserves coverage's benefit (Section~\ref{sec:natural-robust}). The new-family confirmation reports a \ConfDifference{}-percentage-point gain over generic ordering; post-hoc prefixes and leave-one-family-out descriptions explain its concentration while retaining that main result. Equal detection counts identify coverage, rather than revision-specific weighting, as the effective tested component. Native full-constraint verifier outcomes belong to their separate qualification track. Development-calibrated model verdicts offer a different detection/false-rejection trade-off, while ledger checking attaches a replayable witness to each rejection.

The scientific decision object is a fixed analytical query under an approved policy revision and knowledge cutoff. Recording, replay and integrity machinery bind each decision to that request, its legal witnesses and its independently checked outputs.

\section{Conclusion}
Versioned business semantics require revalidation of executable analytics against explicitly requested definitions. Legal witnesses, knowledge cutoffs and revision binding separate present-answer agreement from query conformance. In the authored development suite, obligation coverage detects discrepancies missed by generic ordering, while current-only and equal-weight controls match the priority-weighted result. The old/new matrix additionally separates semantic changes from metadata changes, unrelated definitions and valid historical requests. The two local-model studies expose definition errors that are invisible on the current snapshot and show how legal witnesses detect them in deterministic and sampled queries. The target-definition-only pilot control matches the revision-informed coverage result at the evaluated budget. Calibrated judge configurations exhibit a separate detection/false-rejection trade-off, whereas ledger checking provides replayable evidence for each rejection. The preregistered confirmation on \ConfFamilies{} new authored families detects \ConfGuided{} of \ConfBad{} discrepancies versus \ConfGeneric{} for generic ordering, with zero false rejection among \ConfAgree{} agreeing cells and a family sign-flip $p=\ConfP$. Current-obligation and equal-weight scoring retain the same detection. Together these results establish an executable, auditable revalidation profile and identify obligation coverage as the component responsible for the observed detection benefit. Its empirical estimates apply to the declared authored schemas, legal domains, producer configurations and verification budgets.

\appendix
\section{Prospective External Extension}
\label{app:external-extension}
The local confirmation of Section~\ref{sec:confirmation} completes the authored-family component. The external-provider, external-schema and native full-constraint comparisons below remain unexecuted. Their model sources, schemas, sample sizes and operational resources must be frozen before their respective collection.
\subsection{Questions and strata}
The empirical extension asks whether change conditioning improves detection at comparable false rejection, abstention, and computation; whether historical requests and non-semantic controls are handled correctly; whether gains transfer to unseen combinations/schemas; and what total lifecycle cost is.

The planned candidates will be frozen outputs from two independent strong-model providers and at least one open-weight model, extending the two local bundles of the pilot. All will receive complete policy and explicit revision and will not be asked to make mistakes. Generation/repair allowances will be specified prospectively and all attempts retained.

EntSQL, Spider 2.0, and BIRD can supply independent schemas and contexts~\cite{liao2026entsql,lei2025spider2,li2023bird}. Researcher-written version policies are identified as overlays rather than native gold. License and execution availability precede inclusion. Parent schemas and policy combinations separate development and held-out tests.

\subsection{Matched strong baselines}
Primary controls include legal random/property tests, full-constraint plan-aware or verification tests, and static checks plus metric execution plus generic witnesses. Syntax success is diagnostic, not the decisive baseline. GROUND-style checks are adaptations unless the original system is reproduced; PV-SQL verifier-only checks are not its full pipeline~\cite{tian2026pvsql}.

Reference-requiring tools receive a reference from the same definition. No arm receives richer join rules, time conventions, or retries. Unsupported constructs/timeouts remain in support matrices. An upstream equality response retains its documented testing or proof scope.

\subsection{Endpoints and statistics}
Detection is rejection among independently labeled unequal queries. False rejection is measured among queries conformant within the identical profile. Unrefuted errors, method abstention, and unresolved labels have separate denominators. Actual coverage and uncertainty accompany calibrated false-rejection targets.

An independent scenario/parent task is the sampling unit. Versions, producer outputs, seeds, and arms remain in one paired cluster. Development, calibration, and test cannot share parents. Cluster-level intervals describe paired differences; repeated executions do not enlarge independent sample size.

Pilot data determine feasibility, labeling coverage, error prevalence, cluster variability, and cost, not confirmation tasks selected because a method wins. Sample size and stopping rules freeze before confirmation outcomes using target precision and approved resources.

\subsection{Costs and ablations}
Costs include generation, solving, legality, invalid candidates, metric/SQL execution, model judgment, retries, failures, and timeouts. Preparation, per-query verification, and revision updates are separate. Model calls record input/output tokens and provider accounting.

The executed development ablations remove changed-policy contrasts from obligation scoring and changed-obligation priority separately, holding the revision-informed bank and tie key fixed. The previous-information-free pilot control is completed in Section~\ref{sec:natural-robust}; independent-schema transfer, business-legality and resource-aware choice ablations remain unexecuted. Query guidance and selective reuse are evaluated only if implemented. Repair is a separate endpoint with full costs and independently adjudicated final answers, not inferred from rejection.

\section{Illustrative Policies}
A January sale of 100 minor currency units has a matching payment. Without refunds, posting/original-period definitions both return 100. A funded February refund of 20 minor units, known by the cutoff, leaves January at 100 under posting attribution and changes it to 80 under approved restatement. This hand-computed example is not a model result.

Tax uses net/tax components; payment policies do not infer partial-payment tax. Event-date FX selects the payment/refund date, whereas sale-date FX selects the associated sale date. Equal-rate and unequal-rate controls avoid testing only obvious differences.

\section{Records and Reproduction}
Records bind origin, parent scenario, policy/revision, interval/cutoff, dialect, ledger, legality, raw output, metric value, decision, failure, and cost. Scientific labels carry their own values and profiles. Missing records cannot be repaired by favorable results.

Development writes new directories with source/input digests. Confirmation freezes protocol and labels before decisions. Clean reproduction checks raw records and full denominators. Public packages allowlist manuscripts and licensed, sanitized artifacts, excluding credentials, client data, private architecture, and submission planning. The locally retained evidence includes generated pilot, robustness and confirmation tables, archived local-model output, frozen banks, independent labels, and checker results. The public source candidate contains the manuscript dependency closure; raw experimental material requires a separately reviewed distribution package. No public artifact identifier has been assigned.

\subsection{Controlled protocol and resource profile}
The revised controlled bank deduplicates public base snapshots, policy/request-generated event templates and sixteen ledgers from the original fixed seed. Reference and mutation partners see identical public information and selected ledgers. No arm sees mutation tags, fixture families, candidate SQL or scientific values during selection; generation also cannot access authored witness ledgers. Getter-trap and metadata-perturbation regressions test this boundary. The dated and monetary templates remain development choices, not a new sampling distribution of real tasks.

Each SQL invocation starts a new Python process with an ephemeral SQLite database. The helper admits SELECT/CTE queries, enables query-only access, applies a native authorizer and a function whitelist, and limits virtual-machine steps. The recorded configuration uses a 1,500\,ms process limit and 100,000 VM steps. Multi-statement, database attachment, environment access, non-scalar, NULL, unsupported numeric range, and resource failures return explicit unresolved results. These controls qualify the research executor, not an arbitrary production deployment.

All eight selected tests are executed in each of the four bank-order arms; finding an early discrepancy does not remove later costs. The snapshot arm executes one test and is reported separately. The script records selection time, each policy/SQL invocation, and common bank preparation. Scientific label construction is additional evaluation overhead, not free operational preparation.

\subsection{Independent check and provenance}
Before collection, the runner captures the source closure and refuses to finalize evidence if it changes. The checker verifies identities, full Cartesian label coverage, selected-input coverage, source bindings, freeze ordering, and counts. It reruns the separate policy evaluator and SQL, replays public selection, and independently reaggregates overlap, missed witnesses and budget curves. Obligation coverage uses the operational policy interpreter and is therefore a manipulation check, not independent scientific evidence of policy correctness. Primary labels are recomputed from the full suite and rejections require a selected legal discrepancy.

Deletion, duplication, value tampering, malformed execution accounting, and a witness outside the selected subset are checked as integrity regressions. A failed or existing receipt is never overwritten. Derived manuscript counts require a current passing receipt and matching raw/source digests; evidence collection and document generation are separate steps.

\section{Business Policy and Schema}
\begin{table}[htbp]
\caption{Implemented policy fields. Revision and request applicability are independent of representation-format versions.}
\begin{tabularx}{\linewidth}{@{}lY@{}}\toprule
Field & Meaning \\\midrule
Metric / revision & Explicit identifier and strictly matching string or integer revision \\
Basis & Sales recognition or gross payment receipts \\
Refund attribution & Posting period or original sale period for sales; posting for receipts \\
Tax mode & Net or gross sales; no inferred partial-payment tax decomposition \\
FX date & Original sale date or contribution event date \\
Population & Explicit eligible sale statuses \\
Rounding & Positive event half-up in base-currency minor units, then signed summation \\
Interval / cutoff & Half-open report dates and separate occurrence/recording knowledge cutoff \\\bottomrule
\end{tabularx}
\end{table}

Sales contain a transaction key, customer key, occurrence and recording dates, net/tax amounts, currency, and status. Payments/refunds link to sales; rates use a unique currency/date key and positive integer numerator/denominator. Customer identifiers may repeat and keys are unique within their own relations. This schema was authored for the study and does not reproduce a client or company database.

Legality enforces per-sale collected gross amounts, component-wise net/tax refund ceilings, and cumulative funding available by each refund date. The base currency has explicit unit-rate entries. Future stored events remain structurally legal but cannot contribute before occurrence and recording by the requested cutoff. Policy mismatch is an invalid request, not proof that the warehouse is corrupt.

\section{Support Boundaries}
Explicit research policies do not infer authoritative accounting treatment. Multi-column outputs, nondeterminism, unsupported dialects, unspecified time zones, missing rates, and numeric overflow are unresolved until qualified. Empty intervals/populations follow the explicit zero-result rule.

\bibliographystyle{ACM-Reference-Format}
\bibliography{references}
\end{document}